%% file: root.tex
\documentclass[letterpaper, 10 pt, conference]{ieeeconf}  % Comment this line out if you need a4paper

\IEEEoverridecommandlockouts                              % This command is only needed if 
\usepackage{amsmath,mathtools,amsfonts}
\usepackage{bbm}
\DeclareMathOperator*{\argmax}{arg\,max}
\DeclareMathOperator*{\argmin}{arg\,min}
\DeclareMathOperator*{\minimize}{minimize}
\DeclareMathOperator*{\maximize}{maximize}
\DeclareMathOperator*{\st}{subject~to}
\DeclareMathOperator*{\optn}{(optionally)}
\usepackage{bm}
\usepackage{graphicx}
\usepackage{multirow}
\usepackage{booktabs}
\usepackage{algorithm}
\usepackage{algorithmic}
\usepackage{physics}
\usepackage{subfig}
\usepackage{wrapfig}
\usepackage{dsfont}
\usepackage{hyperref}
\usepackage{xcolor}

\newtheorem{proposition}{Proposition}

\makeatletter
\g@addto@macro\normalsize{%
\setlength\abovedisplayskip{3pt}
\setlength\belowdisplayskip{3pt}
\setlength\abovedisplayshortskip{3pt}
\setlength\belowdisplayshortskip{3pt}
}
\makeatother

\title{\LARGE \bf
Team Orienteering Coverage Planning with Uncertain Reward
}

\author{Bo Liu$^{1}$, Xuesu Xiao$^{1}$ and Peter Stone$^{1,2}$
\thanks{$^{1}$Xuesu Xiao, Bo Liu, and Peter Stone are with Department of Computer Science, University of Texas at Austin, Austin, TX 78712 {\tt\scriptsize \{bliu, xiao, pstone\}@cs.utexas.edu}. $^{2}$Peter Stone is also affiliated with Sony AI .}}

\begin{document}

\maketitle
\thispagestyle{empty}
\pagestyle{empty}

\input{content/abstract}

\input{content/intro}
\input{content/related}
\input{content/method}
\input{content/experiment}
\input{content/conclusion}

\bibliographystyle{IEEEtran}
\bibliography{IEEEabrv,references}

\end{document}

%% file: content/abstract.tex
\begin{abstract}
Many municipalities and large organizations have fleets of vehicles that need to be coordinated for tasks such as garbage collection or infrastructure inspection.  Motivated by this need, this paper focuses on the common subproblem in which a team of vehicles needs to plan coordinated routes to patrol an area over iterations while minimizing temporally and spatially dependent costs. In particular, at a specific location (e.g., a vertex on a graph), we assume the cost grows linearly in expectation with an unknown rate, and the cost is reset to zero whenever any vehicle visits the vertex (representing the robot ``servicing" the vertex).
We formulate this problem in graph terminology and call it Team Orienteering Coverage Planning with Uncertain Reward (TOCPUR). We propose to solve TOCPUR by simultaneously estimating the accumulated cost at every vertex on the graph and solving a novel variant of the Team Orienteering Problem (TOP) iteratively, which we call the Team Orienteering Coverage Problem (TOCP). We provide the first mixed integer programming formulation for the TOCP, as a significant adaptation of the original TOP. We introduce a new benchmark consisting of hundreds of randomly generated graphs for comparing different methods. We show the proposed solution outperforms both the exact TOP solution and a greedy algorithm. In addition, we provide a demo of our method on a team of three physical robots in a real-world environment. The code is publicly available at \url{https://github.com/Cranial-XIX/TOCPUR.git}.
\end{abstract}

%% file: content/intro.tex
\section{Introduction}
Mobile agent fleets are now being used for many purposes in our daily life, such as a team of mobile robots delivering food~\cite{sun2019design}, a school bus fleet picking up students, or a garbage truck fleet collecting garbage.

In many such situations, visiting a particular location results in some benefit (e.g. collecting piled up garbage), which we model as a reward. The overall objective is therefore to collect as much reward as possible, while ensuring that each vehicle's travel time stays within some budget. This problem can be formulated as the Team Orienteering Problem (TOP)~\cite{chao1996team}. However, TOP assumes the reward at each location is a known constant before being collected and set to zero. This formulation does not suit problems in which the reward can accumulate over time. For example, consider a garbage truck fleet collecting garbage in a city. The amount of garbage in general grows over time and it becomes much more beneficial to visit a location that has not been visited for a long time. The expected garbage growth rate at different locations might be different and unknown to the agents beforehand. But whenever an agent visits a location and collects the garbage, it can update its estimation of the growth rate at that location. In addition to TOP assuming a known constant reward, in its typical formulation each location can only be visited once, which again significantly limits its application.

\begin{figure}[t]
    \centering
    \includegraphics[width=\columnwidth]{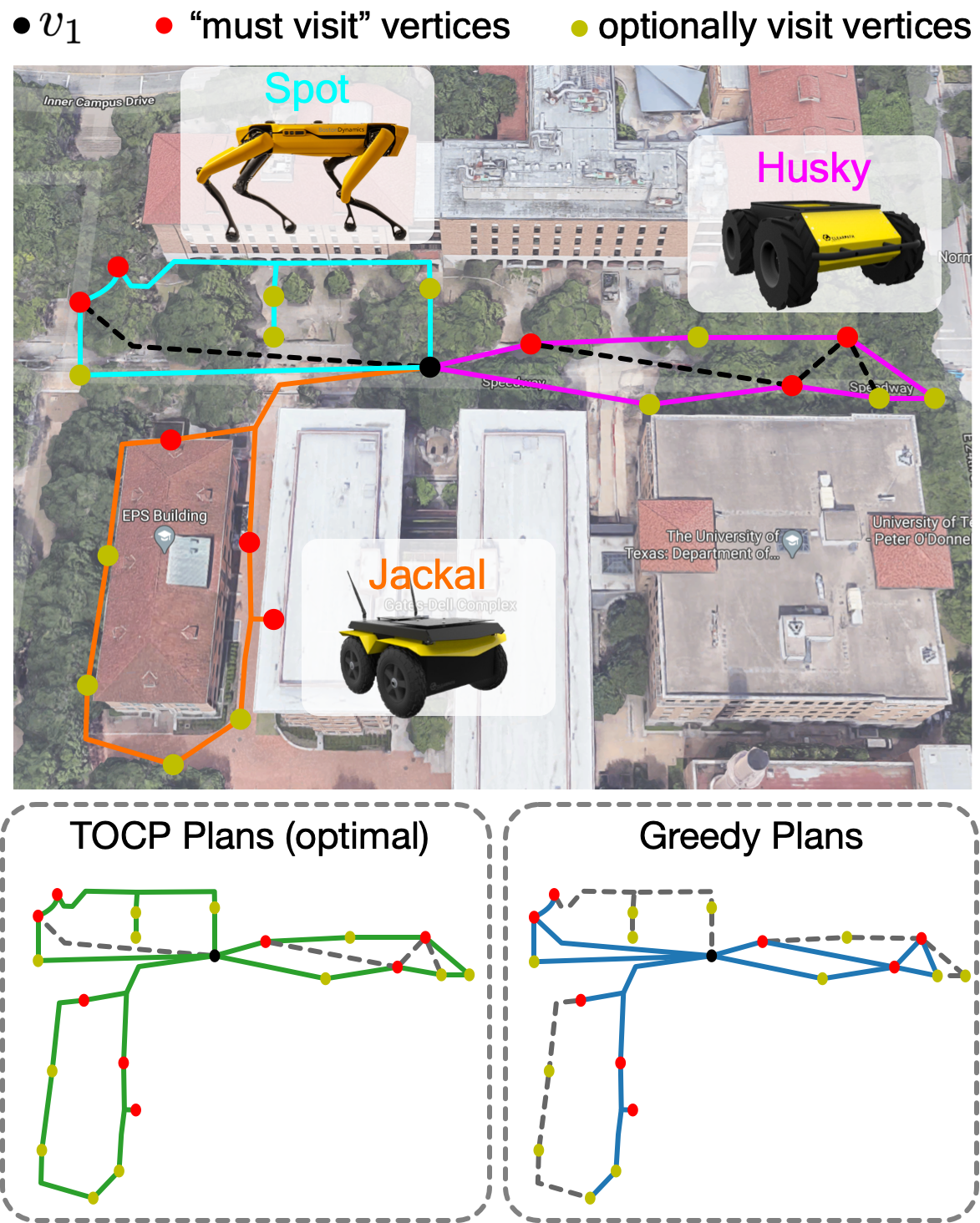}
    \caption{Real-World Demonstration: From the same black vertex ($v_1$), a fleet of three robots is tasked with visiting all red vertices and as many yellow vertices as possible. Within the same budget, TOCP finds the optimal plan that covers all red and yellow vertices \textbf{(lower left)}, while the greedy baseline misses many optional yellow vertices and does a lot of backtracking \textbf{(lower right)}. }
    \label{fig:gdc}
    \vspace{-20pt}
\end{figure}
As a result, in this paper, we introduce a novel coverage planning problem, called Team Orienteering Coverage Planning with Uncertain Reward (TOCPUR).
In TOCPUR, a team of mobile agents is tasked to patrol an area over multiple iterations. The goal is to reduce the time- and place-dependent costs (i.e., negative rewards). In particular, we assume the cost grows between iterations and stays constant within an iteration, and each vertex can be visited multiple times in an iteration but the cost is only reduced once (e.g., the garbage is collected during the first visit on that day where each day is an iteration). Optionally, we allow users to specify a subset of locations that the fleet has to visit. This option is useful when the fleet of agents has a primary task (e.g. routine check at certain locations) and some secondary tasks (e.g. collecting as much garbage as possible).

We solve TOCPUR by simultaneously estimating the unknown and growing costs over the area and solving a novel variant of the TOP, called the Team Orienteering Coverage Problem (TOCP), which allows multiple visits to the same nodes.
In this paper, we refer to TOCP and TOP plans as the solutions to their corresponding problems. We introduce a benchmark of hundreds of randomly generated graphs and show improved performance using the proposed method compared to both the TOP solution and a greedy algorithm. We also demonstrate the proposed planner working on a team of three physical robots in a real-world environment (Fig. \ref{fig:gdc}). 

%% file: content/related.tex
\section{Related Work}
In this section, we briefly review prior literature in orienteering problems and coverage planning.

\paragraph{Orienteering Problem}
The proposed TOCPUR problem is closely related to the team orienteering problem (TOP)~\cite{chao1996team}, which is a multi-tour extension of the orienteering problem (OP)~\cite{golden1987orienteering}. The OP originates from the Travelling Salesman Problem and is heavily studied in the optimization community. In the standard OP, an agent needs to plan a path under a fixed length constraint that collects as much reward as possible as it traverses the vertices along that path. The TOP is then a multi-tour extension of the OP where a team of agents plan consecutively, which is equivalent to a single agent planning multiple times. The major difference between TOCPUR and TOP is that in TOCPUR, we allow each vertex to be visited multiple times, which better resembles many real-world scenarios but makes the problem harder. As a result, the planned route becomes a walk instead of a path (e.g. a walk without loops) for each agent (See Fig.~\ref{fig:intro}).
\begin{figure}
    \centering
    \includegraphics[width=\columnwidth]{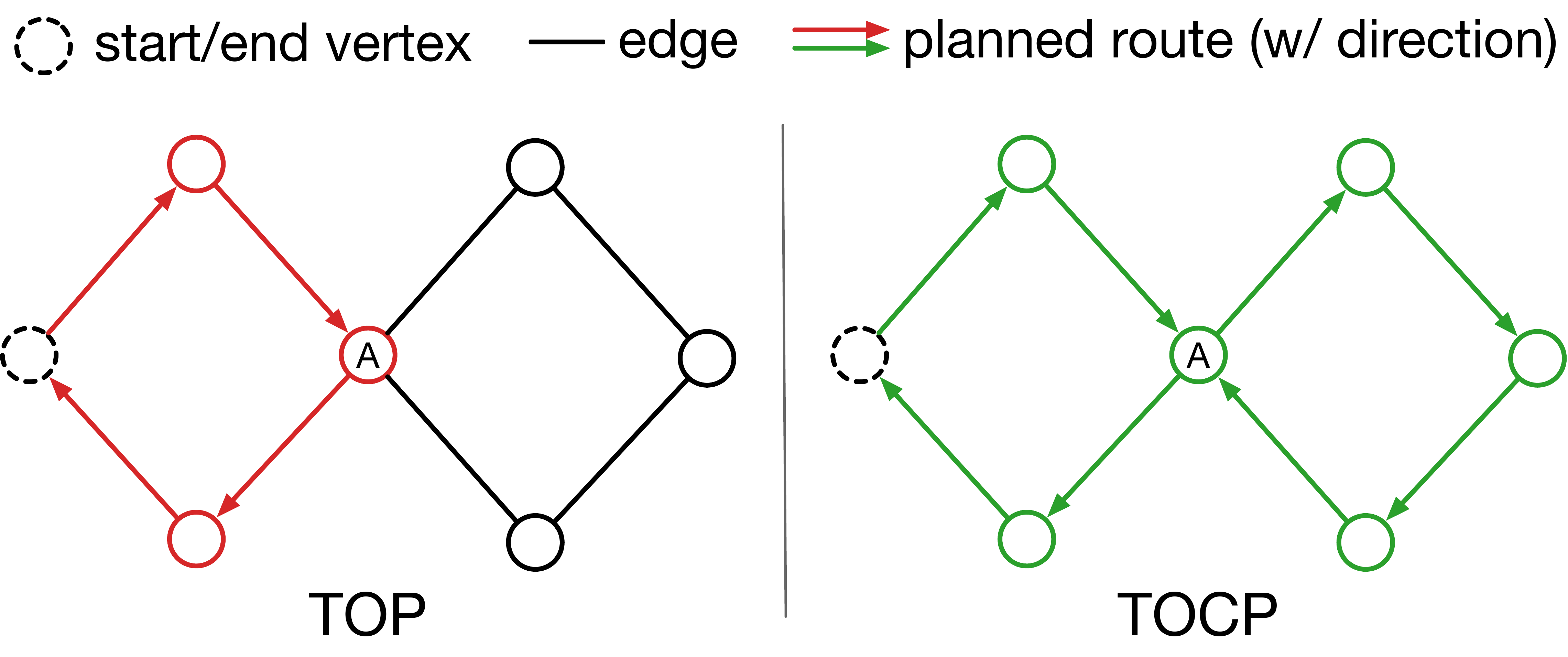}
    \caption{The TOP \textbf{(left)} and the TOCP \textbf{(right)} plans on the same graph instance. In the standard TOP formulation, an agent cannot traverse a vertex twice. Hence in TOP, it is impossible to visit the nodes on the right of A in a closed route that starts/ends at the leftmost vertex.}
    \label{fig:intro}
    \vspace{-15pt}
\end{figure}
While this modification seems to be a small change, the core constraints in the formulation of the standard TOP problem can no longer be used. This modification is also the key contribution of our formulation of the TOCPUR problem. In addition, TOCPUR models cumulative reward over time where at each time step, the reward is sampled identically and independently from a fixed unknown distribution. This setup is similar to the OP with stochastic profits (OPSP)~\cite{ilhan2008orienteering}. However, in OPSP, the objective is to maximize the probability that the total collected profit will be greater than a predefined target value. In addition, OPSP does not allow for reward/profit to grow over time.

Besides TOP and OPSP, there exists a big family of variants of the OP. For instance, prior research has considered the OP with time window (OPTW)~\cite{labadie2012team} where the agent can only visit a vertex within a pre-specified time window. Generalized OP (GOP) extends OP such that the reward at each vertex is a non-linear function with respect to a set of attributes. Multi-agent OP (MAOP)~\cite{chen2014multi} considers a competitive game among multiple agents trying to solve the OP individually. More recently, multi-visit TOP (MVTOP)~\cite{hanafi2020multi} has been proposed where each vertex needs to be served by multiple agents with different skills in a certain order. In general, different combinations of having multiple tours, multiple agents, different time windows, and stochastic rewards have been studied. However, in the formulation of all these variants of the OP, each vertex is only allowed to be visited once by a single vehicle in a single tour. TOCPUR, by allowing multiple visit to a single vertex for a single agent, has much broader applications. Gunawan et al.~\cite{gunawan2016orienteering} provide a comprehensive overview of existing variants of the OP.

\paragraph{Coverage Planning} Coverage planning (CP) is the task in robotics of determining a path to cover all points in an area while avoiding obstacles. Standard CP problems involve both high-level path planning and low-level motion planning. Typical methods will break the free space, i.e., the space free of obstacles, down into simple, non-overlapping regions called cells~\cite{oksanen2009coverage}. Then an exhaustive walk through the graph defined by the decomposed cells is found. TOPCUR is similar to CP problems in that it also attempts to maximize the coverage over all vertices through planned routes. But in contrast, TOPCUR does not require complete coverage of all vertices. In our formulation, TOPCUR primarily focuses on the high-level path planning.

%% file: content/method.tex
\section{Method}
In this section, we first define the TOCPUR problem using graph terminology. Then we propose to optimize the per-iteration objective while updating the reward estimation simultaneously. The per-iteration objective is then specified by a novel mixed integer programming formulation. In addition, we provide a simple greedy method as a baseline for solving the problem approximately.

\subsection{Notation}
We represent the area to be patrolled as a symmetric directed graph $G=(V, E)$, where $V=\{v_1, v_2, \dots, v_N\}$ is the set of $N$ vertices and $E$ is the edge set. $e_{ij}$ denotes an edge from $v_i$ to $v_j$, which has a length of $l_{ij}$.\footnote{Since $G$ is symmetric directed, if $e_{ij}$ exists, $e_{ji}$ also exists and $l_{ij}=l_{ji}$. The symmetric directed assumption represents bidirectional traffic.} In addition, we assume the agent fleet consists of $M$ agents in total, which we call agent $1$ to $M$ when the context is clear. We use $[x]$ to denote $\{1,2,\dots,x\}$, and $[a, b]$ to denote $\{a, a+1, \dots, b\}$.

\subsection{Problem Formulation}
In this section, we define the TOCPUR problem. We consider $H$ iterations of route planning for a vehicle fleet of size $M$ over a symmetric directed graph $G=(V, E)$. Each vertex $v_i$ accumulates a cost of $\kappa_{i,t}$ right before the $t$-th iteration, where $\mathbb{E}[\kappa_{i,t}] = \mu_i^*$ and $\mu_i^*$ is unknown a priori. In other words, the cost at each vertex grows linearly in expectation at an unknown rate (In the garbage truck application, think of an iteration as a day of garbage collection and $\kappa_{i,t}$ as the amount of garbage that appears overnight at location $i$ prior to day $t$). Next, denote $T_{i,t}$ as the most recent iteration prior to $t$ when the vehicle fleet visited $v_i$. Then $v_i$ at iteration $t$ accumulates a total cost of $c_{i,t} = \sum_{k=T_{i,t}+1}^t \kappa_{i,k}$. The entire cost over $G$ at iteration $t$ is specified as $c(G, t) = \sum_{i\in[N]} c_{i,t}$. Intuitively, the total cost at the end of an iteration is the sum of the current accumulated costs of all the vertices \emph{not} visited during that iteration (the current cost of those visited is reset to 0).
The goal is then to plan $M$ routes $\{\tau_{m,t}\}_{m=1}^M$, each no longer than a maximum length $l_\text{max}$, for all vehicles at each iteration $t \in [H]$, such that the total cost on $G$ over the horizon is minimized. Specifically, each route $\tau_{m,t} = (v_1, \dots, v_1)$ is a sequence of vertices that starts and ends in $v_1$. Optionally, we can include a set of ``must visit" vertices $I$ such that the fleet has to visit all vertices in $I$ during each iteration. The whole problem can be described as the following optimization:
\begin{equation}
\label{eq:obj}
\begin{split}
    \minimize_{\{\{\tau_{m,t}\}_{m=1}^M\}_{t=1}^H}&~~~ \mathcal{L} = \sum_{t=1}^H c(G, t),~\text{where} \\
    &~~~ c(G, t) = \sum_{i=1}^N \sum_{k=T_{i,t}+1}^t \kappa_{i,k}.\\
    \st & ~~~\forall m,t~~\sum_{e_{ij} \in \tau_{m,t}} l_{ij} \leq l_\text{max},\\
        & ~~~\forall i,t,~~T_{i,t+1} = \begin{cases}
                                         t           & \text{if}~v_i \in \bigcup_m \tau_{m,t}, \\
                                         T_{i,t}   & \text{otherwise}
                                     \end{cases}\\
          &~~~\forall i,~~T_{i, 1} = 0,\\
    \optn & ~~~ \forall v_i \in I,~~v_i \in \bigcup_m \tau_{m,t}.
\end{split}
\end{equation}
Here, $\bigcup_m \tau_{m,t} = \bigcup_m \{ v~|~v \in \tau_{m,t}\}$ and we abuse the notation a bit to denote $e_{ij} \in \tau$ if $v_i, v_j$ appear consecutively in $\tau$. This optimization problem (Opt.~\ref{eq:obj}) is NP-hard since it can easily reduce to the Travelling Salesman Problem. We emphasize two points: 1) As $\kappa_{i,t}$ is drawn randomly from a distribution with an unknown mean $\mu_i^*$, it is in general impossible to have an optimal open-loop plans for solving Opt.~\ref{eq:obj}. 2) Unlike in TOP, we do not assume that each $v_i$ can only appear once in any $\tau_{m,t}$. However, as we will see in the following, we can safely assume each $e_{ij}$ appears at most once in any $\tau_{m,t}$.

\subsection{Reward Estimation and Per-Iteration Planning}
Solving Opt.~\ref{eq:obj} exactly for large $H$ is computationally intractable. In fact, even when $H=1$, the problem remains NP-hard. In addition, as the true parameters $\{\mu_i^*\}_{i=1}^N$ are hidden, an optimal solution must be closed-loop plans that take past observations (e.g. $\kappa$) into consideration. Therefore, we propose to optimize the cost $c(G, t)$ iteratively while updating the estimates of $\{\mu_i^*\}_{i=1}^M$ simultaneously. Specifically, we keep track of $T_{i,t}$ for each node $v_i$, as well as the total observed cumulative cost at $v_i$, i.e. $C_{i,t} = \sum_{k=1}^{T_{i,t}} \kappa_{i,k}$. The maximum likelihood estimation of $\mu_i^*$ at time $t$ is therefore
\begin{equation}
    \hat{\mu}_{i,t} = \begin{cases}
        \frac{C_{i,t}}{T_{i,t}} = \frac{\sum_{k=1}^{T_{i,t}} \kappa_{i,k}}{T_{i,t}} & T_{i,t} > 0\\
        \mu_\text{default} & T_{i,t} = 0,\\
    \end{cases}
    \label{eq:mu_update}
\end{equation}
where $\mu_\text{default}$ is a default value or a pre-specified value based on prior knowledge when there has been no visits to a node.
Therefore, at each iteration $t$, we can use $\{\hat{\mu}_{i,t}\}_{i=1}^N$ to approximately estimate $c(G, t)$ as $\hat{c}(G,t) = \sum_{i=1}^N \hat{\mu}_{i,t}(t - T_{i,t})$. Following the convention in orienteering problems~\cite{chao1996team}, we formulate the per-iteration optimization as a mixed integer program (MIP). In the following, we will temporarily ignore the subscript $t$.
Let $\{x_{ijm}\}$, $\{y_{im}\}$ and $\{z_i\}$ be the binary decision variables and $i,j \in [N]$ and $m\in[M]$. $x_{ijm}=1$ if agent $m$ traverses $e_{ij}$ and $0$ otherwise. Similarly, $y_{im}=1$ if agent $m$ visits $v_i$ and $z_i=1$ if any agent visits $v_i$. It is sufficient to assume $x_{ijm}$ is binary, as justified by the following proposition.
\begin{proposition}
Any optimal solution of Opt.~\ref{eq:obj} has an equivalent solution where each $e_{ij}$ is traversed at most once.
\end{proposition}
\begin{proof}
 Assume otherwise. Let $e_{ij}$ be the edge that is visited at least twice. If $v_i$ is only visited once, then certainly we can remove one traversal of $e_{ij}$ with no problems. If $v_i$ is visited more than once, then there has to exist a route $p$, i.e. a sequence of vertices, that goes from $v_j$ back to $v_i$. But then, since $G$ is symmetric directed, we can remove both traversals of $e_{ij}$ by reversing the traversal direction of the edges along the route $p$. This change results in visiting the same set of vertices with a shorter length (See Fig.~\ref{fig:edge}).
\end{proof}

\begin{figure}[t]
    \centering
    \includegraphics[width=0.9\columnwidth]{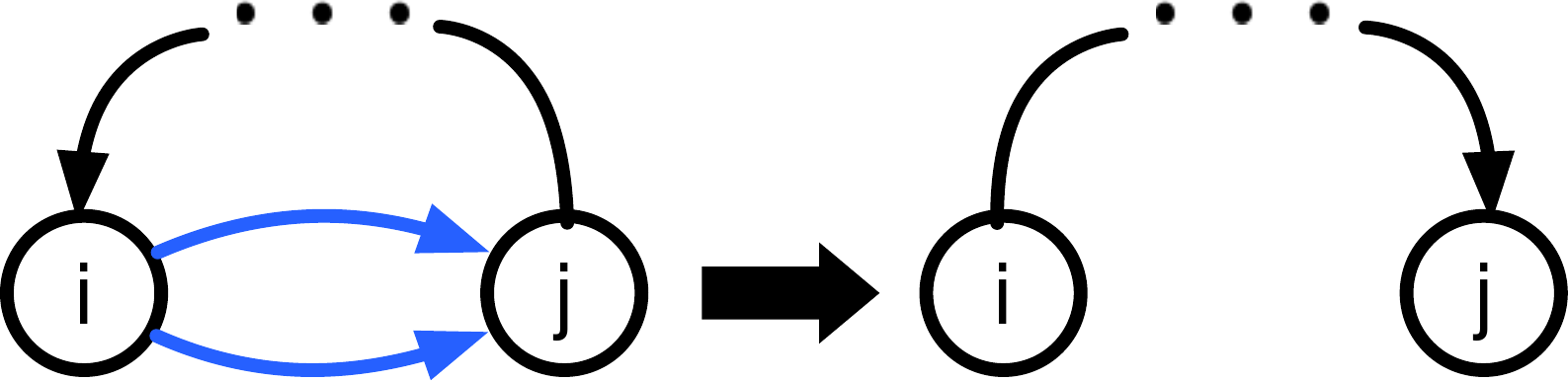}
    \caption{If $e_{ij}$ is visited twice (blue), we can remove them by reversing the direction of another route connecting $v_j$ to $v_i$.}
    \label{fig:edge}
    \vspace{-15pt}
\end{figure}

The per-iteration optimization is then formulated as the following MIP, which we call the team orienteering coverage problem  (TOCP).\footnote{Without futher specification, we assume $m\in[M],i,j\in[N]$.}
\begin{align}
    &                &\maximize~\sum_{m=1}^M \sum_{i=2}^{N} \hat{c}_{i,t} z_{i} \label{opt:0}\\
    &\forall m,      &\sum_{j=2}^N x_{1jm} = \sum_{i=2}^{N} x_{i1m} = 1 \label{opt:1}\\
    &\forall m,i,    &x_{iim}=0 \label{opt:2}
\end{align}
\begin{align}
    &\forall m,i,    &y_{im} \leq \sum_{j=1}^N x_{ijm} \leq \frac{y_{im}\cdot l_\text{max}}{\min_{e_{ij} \in E} l_{ij}} \label{opt:3}\\
    &\forall i,      &z_i \leq \sum_{m} y_{im} \leq z_i \cdot M \label{opt:4}\\
    &\forall m,i,    &\sum_{j=1}^N x_{ijm} = \sum_{j=1}^N x_{jim} \label{opt:5}\\
    &\forall m,      &\sum_{i=1}^N\sum_{j=1}^N l_{ij}x_{ijm} \leq l_\text{max} \label{opt:6}\\
    &\forall m,      &\sum_{j=1}^N u_{1jm} - \sum_{i=1}^N u_{i1m} = \sum_{j=2}^N y_{jm} \label{opt:7}\\
    &\forall m,i>1,  &\sum_{j=1}^N u_{ijm} - \sum_{j=1}^N u_{jim} = y_{im} \label{opt:8}\\
    &\forall m,i,j,  &0 \leq u_{ijm} \leq N \cdot x_{ijm} \label{opt:9}
\end{align}
\begin{equation}
    \optn \qquad \forall v_i \in I,~~z_i = 1 \label{opt:10}
\end{equation}

Eq.~\ref{opt:0} is the objective.
Eq.~\ref{opt:1} ensures all agents start and end in $v_1$.
Eq.~\ref{opt:2} eliminates self-loops.
Eq.~\ref{opt:3}-\ref{opt:4} enforce the definitions of $y_{im}$ and $z_i$. For instance, $y_{im}=0$ if and only if $\forall j,~x_{ijm}=0$.\footnote{The ideal constraint is $y_{im} = \mathds{1}(\sum_j x_{ijm} > 1)$. But due to the need for
constraints to be linear, we represent the constraint with two
inequalities.}
Eq.~\ref{opt:5} ensures the conservation of flow.
Eq.~\ref{opt:6} ensures each agent travels within the length budget $l_\text{max}$.
Eq.~\ref{opt:7}-\ref{opt:9} ensures the found route $\tau_m$ for each agent $m$ is strongly connected. The variables $\{u_{ijm}\}$ define an amount of flow that begins at $v_1$ and is reduced by 1 at every node $m$ visits in sequence. For instance, the net outflow at $v_1$ for $m$ should be $\sum_{j\in[2,N]} y_{jm}$ (Eq.~\ref{opt:7}), the number of vertices $m$ visits in $\tau_m$. If $m$'s trajectory were to include two disconnected components, then there would be no way to consume all of the flow (e.g., Eq.~\ref{opt:7}-\ref{opt:9} would be violated). Note that $u_{ijm}$ should only be positive if $m$ traverses $e_{ij}$, which is ensured by Eq.~\ref{opt:9}.
Finally, Eq.~\ref{opt:10} optionally ensures that all ``must visit" vertices in $I$ are visited.
In the above MIP,~Eq.~\ref{opt:3}-\ref{opt:4} and Eq.~\ref{opt:7}-\ref{opt:9} are \emph{novel} constraints designed specifically for TOCP. The whole algorithm for solving TOCPUR is summarized in Alg.~\ref{alg:psp}.\\\\
\begin{algorithm}[h]
\caption{Reward Estimation and Per-Iteration Planning}
\label{alg:psp}
\begin{algorithmic}[1]
  \STATE \textbf{Maintain:} for each $v_i$, we maintain $T_{i,t}$, the most recent iteration when $v_i$ was visited ($T_{i,1}=0$), and $C_{i,t}$, the observed cumulative cost at $v_i$ up to time $T_{i,t}$.
  \STATE \textbf{Input:} the graph $G=(V, E)$, the ``must visit" vertices $I$, the maximum traversal budget $l_\text{max}$.
  \FOR{$t = 1$ to $H$}
        \STATE $\forall i$,~~update $\hat{\mu}_{i,t}$ according to Eq.~\ref{eq:mu_update}.
        \STATE $\forall i$,~~$\hat{c}_{i,t} = \hat{\mu}_{i,t} \cdot (t - T_{i,t})$.
        \STATE Plan $\{\tau_{m,t}\}_{m=1}^M$ by solving Opt.~\ref{opt:0}-\ref{opt:10}.
        \STATE $\forall i\in[N],~~T_{i,t+1} = 
                \begin{cases}
                    t & \text{if}~ v_i \in \bigcup_m \tau_{m,t}, \\
                    T_{i,t} & \text{otherwise.} \\
                \end{cases}$.
        \STATE $C_{i,t} = \sum_{k=1}^{T_{i,t}} \kappa_{i,k}$.
  \ENDFOR
\end{algorithmic}
\end{algorithm}
\textbf{Remark:} The sub-optimality of Alg.~\ref{alg:psp} originates from two sources: 1) the decomposition sub-optimality caused by decomposing the long horizon planning into per-iteration planning; and 2) the uncertainty sub-optimality caused by inaccurate estimation of $\{\mu_i^*\}$, which decreases as we visit each vertex more often. It is not immediately clear whether the decomposition sub-optimality can ever arise. To illustrate that it can arise, we provide an example such that even when the true $\{\mu_i^*\}$ are provided, the per-iteration optimal plans are still not optimal over the horizon. The example is illustrated in Fig.~\ref{fig:counter-example} with $H=2$. In Fig.~\ref{fig:counter-example}, the leftmost graph shows the initial costs on all vertices. For illustration simplicity, we assume there is no growing cost at this moment (so each vertex only has accumulate an initial cost). Due to the travel length budget $l_\text{max} = 5$, the per-iteration optimal plans achieve a total cost of $12$, while the optimal plans achieve a cost of $9$ over two steps.
\begin{figure*}[h]
    \centering
    \includegraphics[width=\textwidth]{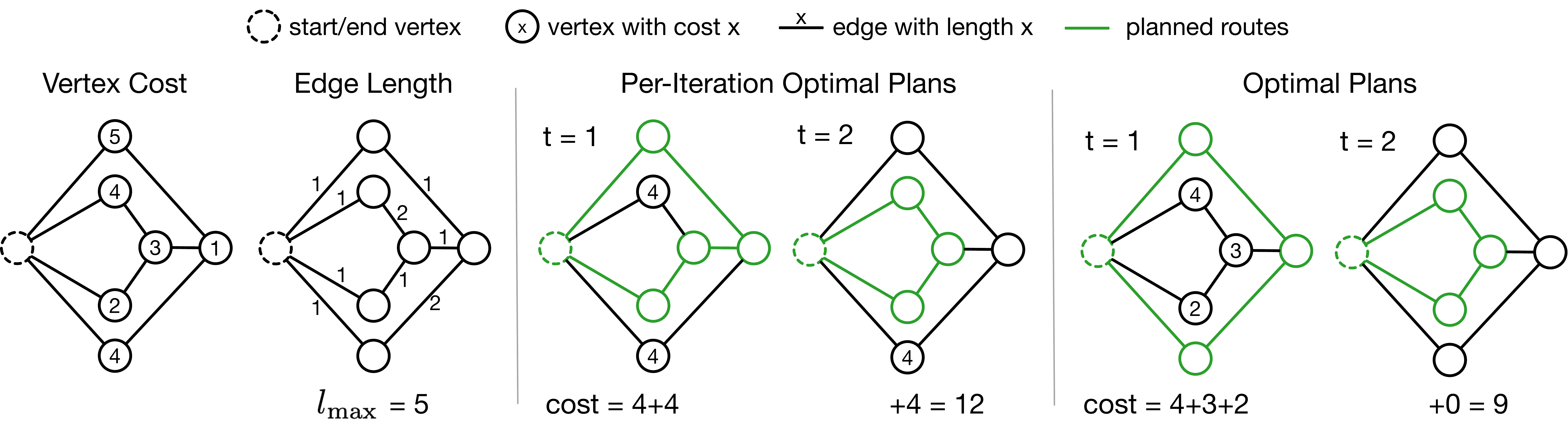}
    \caption{An example problem with $H=2$ and $M=1$, where the per-iteration plans are not optimal. \textbf{left:} the graph $G$. The initial vertex cost is specified on the vertex and the edge length is labeled on each edge. Note that all edges are bi-directional and we assume $l_\text{max} = 5$. \textbf{middle:} the optimal per-iteration plans visits $4$ vertices at $t=1$, and only one of the two un-visited vertices at $t=2$. \textbf{right:} the optimal plans visit 3 nodes at each iteration, covering all nodes in the end.}
    \label{fig:counter-example}
\end{figure*}
\begin{algorithm}[t]
\caption{Greedy Per-Iteration Planning}
\label{alg:greedy}
\begin{algorithmic}[1]
  \STATE \textbf{Input:} the graph $G=(V, E)$, the estimated cost growth for all vertices $\{\hat{\mu}_{i,t}\}_{i=1}^N$, the time since last visit to each node $\{T_{i,t}\}_{i=1}^N$, the ``must visit" vertices $I$, the maximum traversal budget $l_\text{max}$, $D: V\times V \rightarrow \mathbb{R}$, a function that outputs the shortest distance between any pair of vertices (e.g., calculated by Floyd-Warshall).
  \STATE $\forall i \in [N]$,~compute $\hat{c}_{i,t} = \hat{\mu}_{i,t} \cdot (t - T_{i,t})$.
  \STATE $A \leftarrow I$ and $B \leftarrow V\setminus I$. (must/optionally visit vertices)
  \STATE $\forall m\in[M],~\tau_{m,t} = [~],~f_m=0$~(whether $m$ finishes), $v_m = v_1$~(current location of $m$),~$l_m = 0$~(travelled distance of $m$).
  \WHILE{$\sum_m f_m < M$}
    \FOR{agent $m \in \{x \mid f_x = 0, x \in [M]\}$}
        \STATE $Z \leftarrow \{v~|~l_m + D(v_m, v)+D(v, v_1) \leq l_\text{max}\}$.
        \IF{$A \cap Z \neq \emptyset$}
            \STATE $v \leftarrow \argmin_{x \in A\cap Z} D(v_m, v)$;~$X \leftarrow A$.
        \ELSIF{$B\cap Z \neq \emptyset$}
            \STATE $v \leftarrow \argmax_{v \in B\cap Z} \hat{c}_{i,t} / D(v_m, v)$;~$X \leftarrow B$.
        \ELSE
            \STATE $v = v_1$;~$f_m = 1$;~$X \leftarrow \emptyset$. 
        \ENDIF
        \STATE $l_m \leftarrow l_m + D(v_m, v)$.
        \STATE Let $p$ be the shortest path from $v_m$ to $v$.
        \STATE $X \leftarrow X\setminus p$;~$\tau_{m,t}$ append $p$;~$v_m \leftarrow v$.
    \ENDFOR
  \ENDWHILE
\end{algorithmic}
\end{algorithm}
\vspace{-10pt}
\subsection{Greedy Per Step Planning}
In addition to the exact MIP solution from Opt.~\ref{opt:0}-\ref{opt:10}, we provide a greedy algorithm that efficiently and approximately solves Opt.~\ref{eq:obj}, summarized in Alg.~\ref{alg:greedy}. The principle idea is to prioritize must visit vertices first. When all ``must visit" vertices have been visited, we prioritize vertices with a larger ratio between $\hat{c}_{i,t}$ and the distance to $v_i$.

%% file: content/experiment.tex
\section{Experiments}
In this section, we conduct simulated experiments to evaluate two hypotheses: 1) the proposed iterative method outperforms an exact TOP method, and 2) solving the MIP in Opt.~\ref{opt:0}-\ref{opt:10} exactly outperforms the simple greedy algorithm.
\subsection{Simulated Experiments}
To compare the proposed method with the exact TOP method and the greedy method, we introduce a new benchmark that consists of 600 randomly generated graphs. For each graph, the number of vertices $N$ is drawn uniformly at random from $\{10, 12, 14, 16, 18, 20\}$. The vertex $v_1$ is positioned at $(0,0)$ in the standard euclidean plain. For vertices $v_2, \dots, v_N$, each vertex's position is drawn uniformly at random from a $10\times10$ square centered at the origin, with its sides parallel to the x- and y-axes. Then each vertex $v_i$ is connected to its closest $n_i$ vertices, where $n_i$ is drawn uniformly at random from $\{3, 4, 5\}$. $l_\text{max}$ is drawn from $U(20, 20+2N)$, where $U(a,b)$ denotes a uniform distribution from $a$ to $b$. The number of agents $M$ is drawn uniformly at random from $\{2,3,4,5\}$. The number of ``must visit" vertices, $N_I$, is set to $\min(X, M)$, where $X$ is drawn uniformly at random from $\{1,2,3\}$. Denote $V_\text{reachable}$ as all vertices reachable from $v_1$ within the travel budget $l_\text{max}$; we sample $N_I$ vertices from $V_\text{reachable}$ without replacement to form $I$. Finally, for each $v_i$, the expected growth of cost $\mu_i^*$ is drawn from $U(0.1, 0.9)$. For each iteration $t$, the actual growth $\kappa_{i,t} = \min(\max(x_{i,t},~0),~1)$, where $x_{i,t} \sim \mathcal{N}(\mu_i^*, 0.1)$. Finally, we generate 120 random graphs with random seeds ranging from $[1,120]$, following the above procedure for each horizon $H$ in $\{2, 4, 6, 8, 10\}$, which results in a total of 600 random graphs.
\begin{figure}[ht]
    \centering
    \includegraphics[width=\columnwidth]{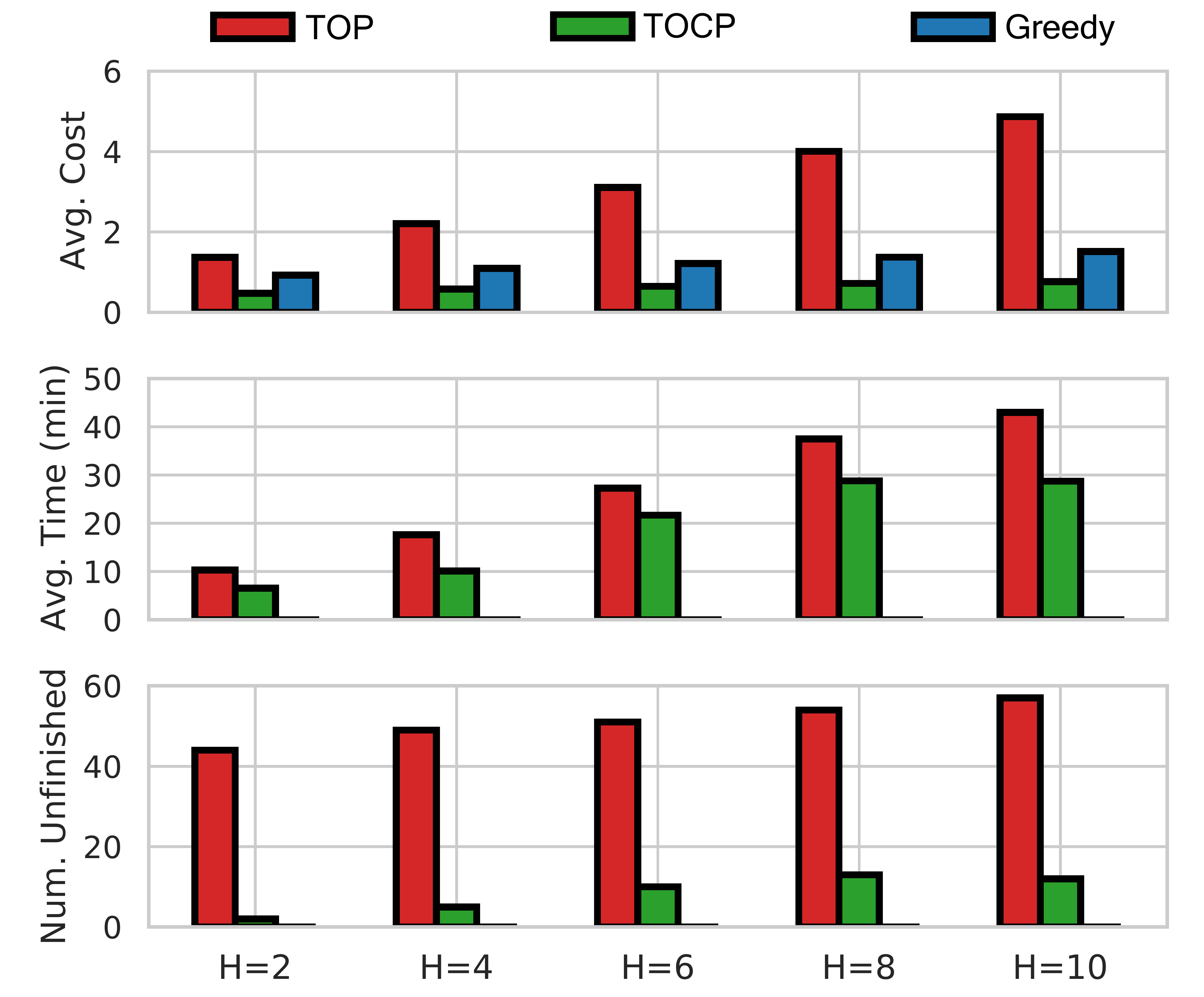}
    \caption{Comparison among TOCP, TOP, and the greedy algorithm on 600 randomly generated graphs.}
    \label{fig:simulated}
    \vspace{-15pt}
\end{figure}

We summarize the experiment results in Fig.~\ref{fig:simulated}. In all cases, we follow Alg.~\ref{alg:psp}, and varying the per-iteration method (line 6 of Alg.~\ref{alg:psp}). We compare TOCP solutions with exact TOP solutions and the solutions found by Alg.~\ref{alg:greedy}. We limit the per-iteration computation time of TOP and TOCP to 1000 seconds. Therefore, occasionally TOP and TOCP might not find a feasible solution. The first row of Fig.~\ref{fig:simulated} shows the total cumulative cost for each horizon $H$, averaged over the subset of 120 random graphs where all methods find a solution. The total cumulative cost is essentially the objective in Eq.~\ref{eq:obj}. TOCP solutions outperform the ones found by the greedy method, and both outperform TOP solutions by a large margin. We provide additional pairwise independent-samples T-tests in Table~\ref{tab:t-test} and highlight the conclusions that are statistically significant ($p \leq 0.05$).
\begin{table}[ht]
\centering{}%
\begin{tabular}{l|cc|cc}
         & \multicolumn{2}{c}{TOP-TOCP} & \multicolumn{2}{c}{Greedy-TOCP} \\
 Horizon & t-score & p-value & t-score & p-value\\
\hline
$H=2$  & 3.63 & \textbf{.0004}  & 2.38 & \textbf{.018} \\
$H=4$  & 4.00 & \textbf{.0001}  & 2.04 & \textbf{.043} \\
$H=6$  & 4.36 & \textbf{.00004} & 1.97 & \textbf{.050} \\
$H=8$  & 4.46 & \textbf{.00003} & 1.95 & .054 \\
$H=10$ & 4.48 & \textbf{.00003} & 1.96 & .052 \\
\hline
\end{tabular}
\caption{Independent-samples T-tests for TOP v.s. TOCP and the greedy method v.s. TOCP.}
\label{tab:t-test}
\vspace{-10pt}
\end{table}
\begin{figure*}[t]
    \centering
    \includegraphics[width=0.95\textwidth]{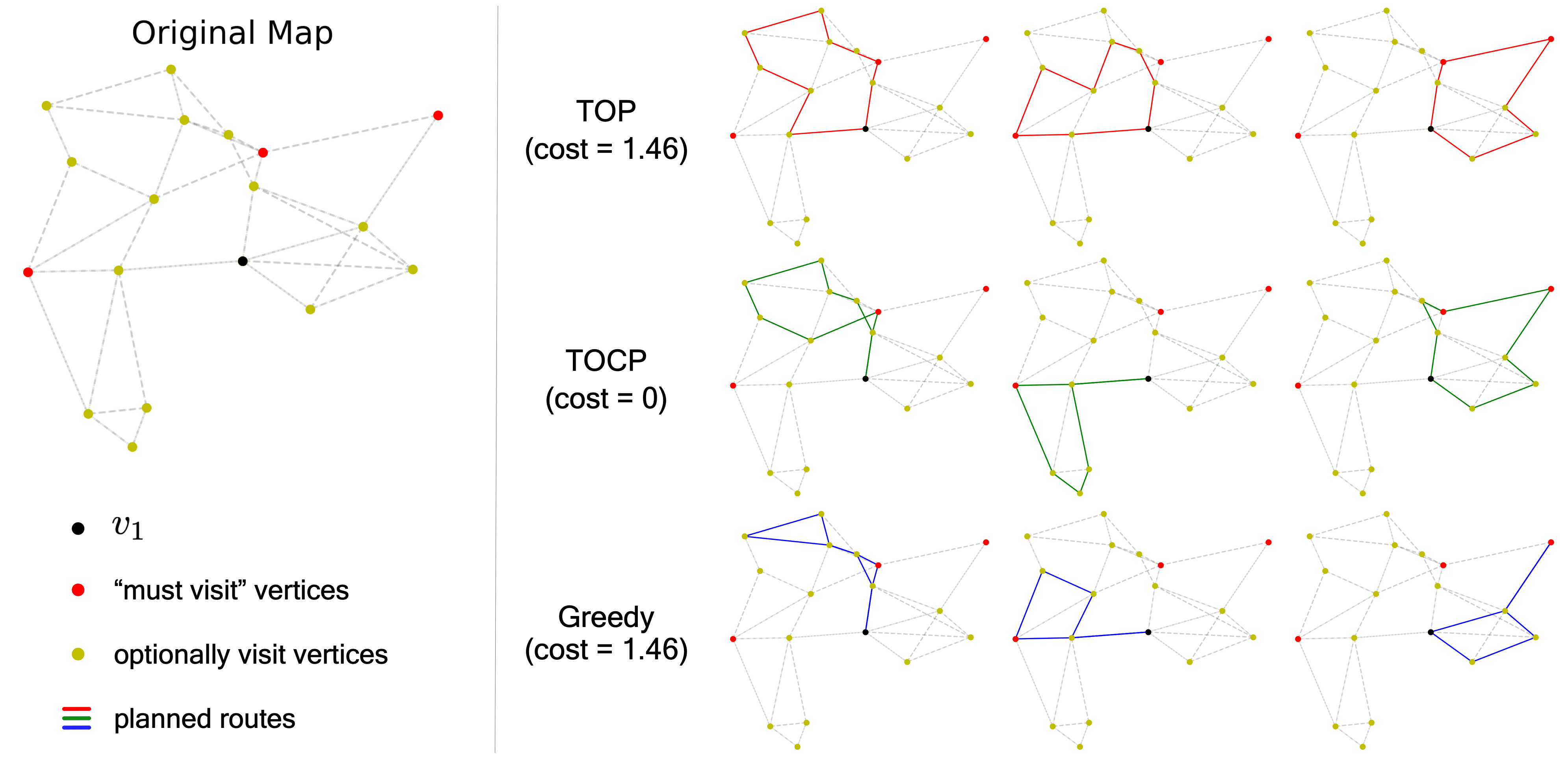}
    \caption{Solutions from TOP, TOCP and the greedy method on an example graph with $H=1$ and $M=3$.}
    \label{fig:visualization}
    \vspace{-10pt}
\end{figure*}

The second row in Fig.~\ref{fig:simulated} reports the average computation time, again over the subset of graphs where all methods find solutions. As expected, TOP and TOCP take much longer than the greedy method does. The last row of Fig.~\ref{fig:simulated} reports the number of graphs for which each method fails to find a solution. Since TOP does not allow a second visit to any vertex, TOP fails more often than TOCP.

In addition to the quantitative evaluations in Fig.~\ref{fig:simulated}. We also provide a qualitative visualization with $H=1$ and $M=3$ in Fig.~\ref{fig:visualization} to further showcase the difference in the three methods used. Among all three methods, TOCP is the only one that cover all vertices, thus clearing all the costs.

\subsection{Physical Demo}
We additionally record a demo of applying Alg.~\ref{alg:psp} on three physical robots with $H=1$ in a real-world environment.\footnote{The video link is at~\url{https://drive.google.com/file/d/1pwE-zLbpcYK2DGeWZ2L5ePZsuO78KCpc/view?usp=sharing}.}

%% file: content/conclusion.tex
\section{Conclusion}
In this work, we formulate a novel variant of the team orienteering problem (TOP) that allows multiple visits to the same vertex and uncertain cumulative costs on each vertex over a horizon. We propose a method to iteratively find the per-iteration optimal plans using a novel mixed integer programming formulation based on the maximum likelihood estimates of each vertex's costs. The simulated experiments show that the proposed method greatly outperforms the exact TOP solution. We also provide a real-world demo of the proposed method on three physical robots. In this paper, we focus on high-level route planning. An interesitng direction for future work is to incorporate obstacle avoidance into TOCPUR.

\section*{ACKNOWLEDGMENT}
This work has taken place in the Learning Agents Research Group (LARG) at UT Austin.  LARG research is supported in part by NSF (CPS-1739964, IIS-1724157, NRI-1925082), ONR (N00014-18-2243), FLI (RFP2-000), ARO (W911NF-19-2-0333), DARPA, Lockheed Martin, GM, and Bosch.  Peter Stone serves as the Executive Director of Sony AI America and receives financial compensation for this work. The terms of this arrangement have been reviewed and approved by the University of Texas at Austin in accordance with its policy on objectivity in research. We thank Harel Yedidsion and Yuqian Jiang for their thoughtful discussion on designing the greedy algorithm.